\newif\ifarxiv
\arxivtrue 

\documentclass{llncs}
 
\usepackage{amsmath}
\usepackage{amssymb}
\usepackage{enumitem}
\usepackage{mathtools}
\usepackage{xcolor,colortbl}
\usepackage{xfrac}
\usepackage{tabularx,tabu}
\usepackage{enumerate}
\usepackage{stmaryrd}
\usepackage{algorithm,algpseudocode}
\usepackage{xspace}
\usepackage{xfrac}
\usepackage{siunitx}
\DeclareSIUnit{\microsecond}{\SIUnitSymbolMicro s}

\usepackage{caption}
\usepackage{subcaption}
\usepackage{wrapfig}

\usepackage{bbm}
\usepackage[b]{esvect}
 \usepackage{tikz,pgf}
 \usetikzlibrary{positioning,calc,backgrounds,automata}
 \usepackage[bb=boondox]{mathalfa}
 \usepackage{wrapfig}
 \usepackage{siunitx}
 \usepackage{array}
 \usepackage{multirow}

 \newtheorem{assumption}{Assumption}
 \newtheorem{prob}{Problem}

\newcommand{\tup}[1]{\left(#1 \right)}
\newcommand{\set}[1]{\lbrace #1\rbrace}
\newcommand{\pr}{\mathbb{P}}

\newcommand{\expe}{\mathbb{E}}
\newcommand{\sem}[1]{\llbracket #1\rrbracket}

 \newcommand{\twodots}{\mathinner {\ldotp \ldotp}}

 \newcommand{\cyl}[1]{\mathit{Cyl}(#1)}
 \newcommand{\pref}{\prec}
 \newcommand{\init}{\mathsf{in}}
 
 \newcommand{\true}{\mathit{true}}
 \newcommand{\false}{\mathit{false}}

\newcommand{\seq}[1]{\vv{#1}}

\newcommand{\vbse}{BSE\xspace}
\newcommand{\quanvbse}{Quant\vbse}
\newcommand{\qualvbse}{Qual\vbse}

\newcommand{\mc}{\mathcal{M}}
\newcommand{\Q}{Q}

\newcommand{\trm}{M}
\newcommand{\paths}[2][*]{Q^#1(#2)}
\newcommand{\initd}{\lambda}

\newcommand{\taumix}{\ensuremath{\tau_{\mathsf{mix}}}}

\newcommand{\Out}{\Sigma}
\newcommand{\out}{w}
\newcommand{\rout}{W}
\newcommand{\OutFn}{\ell}
\newcommand{\outpaths}[2][*]{\Sigma^#1(#2)}

\newcommand{\aut}{\mathcal{A}}

\newcommand{\st}{\pi}

\newcommand{\monitor}{\mathcal{A}}
\newcommand{\verdict}{\Lambda}

\newcommand{\transrel}{T}

\algnewcommand{\IfThenElse}[3]{
  \State \algorithmicif\ #1\ \algorithmicthen\ #2\ \algorithmicelse\ #3}
  
\newcommand{\specset}{\Lambda}

\title{Monitoring Algorithmic Fairness\\ under Partial Observations}

\author{Thomas A.\ Henzinger \and 
		Konstantin Kueffner \and
		Kaushik Mallik}
		
\institute{Institute of Science and Technology Austria (ISTA)}

\begin{document}

\newpage
	\maketitle
	
	\begin{abstract}
		    As AI and machine-learned software are used increasingly for making decisions that affect humans, it is imperative
    that they remain fair and unbiased in their decisions. To complement
    design-time bias mitigation measures, runtime verification
    techniques have been introduced recently to monitor the algorithmic fairness of 
    deployed systems. Previous monitoring techniques assume full
    observability of the states of the (unknown) monitored system. Moreover, they
    can monitor only fairness properties that are specified as arithmetic
    expressions over the probabilities of different events. In this work, we extend fairness monitoring to systems modeled as partially observed Markov
    chains (POMC), and to
    specifications containing arithmetic expressions over the expected values of
    numerical functions on event sequences. The only assumptions
    we make are that the underlying POMC is aperiodic and starts in the stationary distribution, with a bound on its mixing time being known. These
    assumptions enable us to estimate a given property for the entire
    distribution of possible executions of the monitored POMC, by observing only a single execution. Our
    monitors observe a long run of the system and, after each new
    observation, output updated PAC-estimates of how fair or biased the
    system is. The monitors are computationally lightweight and, using a prototype
    implementation, we demonstrate their effectiveness on several real-world
    examples.
	\end{abstract}


\section{Introduction}

Runtime verification complements traditional static verification techniques, by offering lightweight approaches for verifying properties of systems from a single long observed execution trace \cite{bartocci2018lectures}.
Recently, runtime verification was used to monitor biases in machine-learned decision-making softwares \cite{albarghouthi2019fairness,henzinger2023monitoring,henzinger2023runtime}.
Decision-making softwares are being increasingly used for making critical decisions affecting humans; example areas include judiciary \cite{chouldechova2017fair,dressel2018accuracy}, policing \cite{ensign2018runaway,lum2016predict}, and banking \cite{liu2018delayed}.
It is  important that these softwares are unbiased towards the protected attributes of humans, like gender and ethnicity.
However, they were shown to be biased on many occasions in the past
\ifarxiv \cite{dressel2018accuracy,lahoti2019ifair,obermeyer2019dissecting,scheuerman2019computers,seyyed2020chexclusion}.
\fi
While many offline approaches were proposed for mitigating such biases 
\ifarxiv 							\cite{bellamy2019ai,wexler2019if,bird2020fairlearn,zemel2013learning,jagielski2019differentially,konstantinov2022fairness},
\fi
 runtime verification introduces a new complementary tool to oversee \emph{algorithmic fairness} of deployed decision-making systems \cite{albarghouthi2019fairness,henzinger2023monitoring,henzinger2023runtime}.
In this work, we extend runtime verification to monitor algorithmic fairness for a broader class of system models and a more expressive specification language.

Prior works on monitoring algorithmic fairness assumed that the given system is modeled as a Markov chain with unknown transition probabilities but with fully observable states \cite{albarghouthi2019fairness,henzinger2023monitoring}.
A sequence of states visited by the Markov chain represents a (randomized) sequence of events generated from the interaction of the decision-making agent and its environment.
The goal is to design a monitor that will observe one such long sequence of states, and, after observing every new state in the sequence, will compute an updated PAC-estimate of how fair or biased the system is.

In the prior works, the PAC guarantee on the output hinges on the full observability and the Markovian structure of the system \cite{albarghouthi2019fairness,henzinger2023monitoring,henzinger2023runtime}.
While this setup is foundational, it is also very basic, and is not fulfilled by many real-world examples.
Consider a lending scenario where at every step a bank (the decision-maker) receives the features (e.g., the age, gender, and ethnicity) of a loan applicant, and decides whether to grant or reject the loan.
To model this system using the existing setup, we would need to assume that the monitor can observe the full state of the system which includes all the features of every applicant.
In reality, the monitor will often be a third-party system, having only partial view of the system's states.

We address the problem of designing monitors when the systems are modeled using partially observed Markov chains (POMC) with unknown transition probabilities. 
The difficulty comes from the fact that a random observation sequence that is visible to the monitor may not follow a Markovian pattern, even though the underlying state sequence is Markovian.
We overcome this by making the assumption that the POMC starts in the stationary distribution, which in turn guarantees a certain uniformity in how the observations follow each other.
We argue that the stationarity assumption is fulfilled whenever the system has been running for a long time, which is suitable for long term monitoring of fairness properties.
With the help of a few additional standard assumptions on the POMC, like aperiodicity and the knowledge of a bound on the mixing time, we can compute PAC estimates on the degree of algorithmic fairness over the distribution of all runs of the system from a single monitored observation sequence.

Besides the new system model, we also introduce a richer specification language---called bounded specification expressions (\vbse).
\vbse-s can express many common algorithmic fairness properties from the literature, such as demographic parity \cite{dwork2012fairness}, equal opportunity \cite{hardt2016equality}, and disparate impact \cite{feldman2015certifying}.
Furthermore, \vbse-s can express new fairness properties which were not expressible using the previous formalism \cite{albarghouthi2019fairness,henzinger2023monitoring}.
In particular, \vbse-s can express quantitative fairness properties, including fair distribution of expected credit scores and fair distribution of expected wages across different demographic groups of the population; details can be found, respectively, in Ex.~\ref{ex:social fairness in lending} and \ref{ex:fair salary distribution in hiring} in Sec.~\ref{sec:BSE}.

The building block of a \vbse is an atomic function, which is a function that assigns bounded numerical values to observation sequences of a particular length.
Using an atomic function, we can express weighted star-free regular expressions (every word satisfying the given regular expression has a numerical weight), average response time-like properties, etc.
A \vbse can contain many different atomic functions combined together through a restricted set of arithmetic, relational, and logical operations.
We define two fragments of \vbse-s: 
The first one is called \quanvbse, which contains only arithmetic expressions over atomic functions, and whose semantic value is the expected value of the given expression over the distribution of runs of the POMC.
The second one is called \qualvbse, which turns the \quanvbse expressions into boolean expressions through relational (e.g., whether a \quanvbse expression is greater than zero) and logical operators (e.g., conjunction of two relational sentences), and whose semantic value is the expected truth or falsehood of the given expression over the distribution of runs of the POMC.

For any given \vbse, we show how to construct a monitor that observes a single long observation sequence generated by the given POMC with unknown transition probabilities, and after each observation outputs an updated numerical estimate of the actual semantic value of the \vbse for the observed system.
The heart of our approach is a PAC estimation algorithm for the semantic values of the atomic functions.
The main difficulty stems from the statistical dependence between any two consecutive observations, which is a side-effect of the partial observability of the states of the Markov chain, and prevents us from using the common PAC bounds that were used in the prior works that assumed full observability of the POMC states \cite{albarghouthi2019fairness,henzinger2023monitoring}.
We show how the problem can be cast as the statistical estimation problem of estimating the expected value of a function  over the states of a POMC which satisfies a certain bounded difference property.
This estimation problem can be solved using a version of McDiarmid's concentration inequality \cite{paulin2015concentration}, for which we need the additional assumptions that the given POMC is aperiodic and that a bound on its mixing time is known.
We use McDiarmid's inequality to find the PAC estimate of every individual atomic function of the given \vbse.
The individual PAC estimates can then be combined using known methods to obtain the overall PAC estimate of the given \vbse \cite{albarghouthi2019fairness}.

Our monitors are computationally lightweight, and produce reasonably tight PAC bounds of the monitored properties.
Using a prototype implementation, we present the effectiveness of our monitors on two different examples.
On a real-world example, we showed how our monitors can check if a bank has been fair in giving loans to individuals from two different demographic groups in the population, and on an academic example, we showed how our monitors' outputs improve as the known bound on the mixing time gets tighter.

\ifarxiv
	The proofs of the technical claims can be found in the appendices.
\fi

\subsection{Related Work}
There are many works in AI and machine-learning which address how to eliminate or minimize decision biases in learned models through improved design principles
\ifarxiv \cite{mehrabi2021survey,dwork2012fairness,hardt2016equality,kusner2017counterfactual,kearns2018preventing,sharifi2019average,bellamy2019ai,wexler2019if,bird2020fairlearn,zemel2013learning,jagielski2019differentially,konstantinov2022fairness}.
\fi
In formal methods, too, there are some works which statically verify absence of biases of learned models
\ifarxiv \cite{albarghouthi2017fairsquare,bastani2019probabilistic,sun2021probabilistic,ghosh2020justicia,meyer2021certifying,john2020verifying,balunovic2021fair,ghosh2021algorithmic}.
\fi
All of these works are static interventions and rely on the availability of the system model, which may not be always true.

Runtime verification of algorithmic fairness, through continuous monitoring of decision events, is a relatively new area pioneered by the work of Albarghouthi et al.~\cite{albarghouthi2019fairness}.
We further advanced their idea in our other works which appeared recently \cite{henzinger2023monitoring,henzinger2023runtime}.
In those works, on one hand, we generalized the class of supported system models to Markov chains and presented the new Bayesian statistical view of the problem \cite{henzinger2023monitoring}.
On the other hand, we relaxed the time-invariance assumption on the system \cite{henzinger2023runtime}.
In this current paper, we limit ourselves to time-invariant systems but extend the system models to partially observed Markov chains and consider the broader class of \vbse properties, which enables us to additionally express properties whose values depend on observation sequences.

Traditional runtime verification techniques support mainly temporal properties and employ finite automata-based monitors 
\ifarxiv
\cite{stoller2011runtime,junges2021runtime,faymonville2017real,maler2004monitoring,donze2010robust,bartocci2018specification,baier2003ctmc}.
\fi
In contrast, runtime verification of algorithmic fairness requires checking statistical properties, which is beyond the limit of what automata-based monitors can accomplish.
Although there are some works on quantitative runtime verification using richer types of monitors (with counters/registers like us) \cite{finkbeiner2002collecting,henzinger2020monitorability,otop2019quantitative,henzinger2021quantitative}, the considered specifications usually do not extend to statistical properties such as algorithmic fairness.

Among the few works on monitoring statistical properties of systems, a majority of them only provides asymptotic correctness guarantees \cite{ferrere2019monitoring,waudby2021time}, whereas we provide anytime guarantees.
On the other hand, works on monitoring statistical properties with finite-sample (nonasymptotic) guarantees are rare and are restricted to simple properties, such as probabilities of occurrences of certain events \cite{bartolo2021towards} and properties specified using certain fragments of LTL \cite{ruchkin2020compositional}.
Monitoring POMCs (the same modeling formalism as us) were studied before by Stoller et al.~\cite{stoller2012runtime}, though the setting was a bit different from ours.
Firstly, they only consider LTL properties, and, secondly, they assume the system model to be known by the monitor.
This way the task of the monitor effectively reduces to a state estimation problem from a given observation sequence.

Technique-wise, there are some similarities between our work and the works on statistical model-checking \cite{ashok2019pac,younes2002probabilistic,clarke2011statistical,david2013optimizing,agha2018survey} in that both compute PAC-guarantees on satisfaction or violation of a given specification.
However, to the best of our knowledge, the existing statistical model-checking approaches do not consider algorithmic fairness properties.


\section{Preliminaries}

\subsection{Notation}

We write $\mathbb{R}$, $\mathbb{R}^+$, $\mathbb{N}$, and $\mathbb{N}^+$ to denote the sets of real numbers, positive real numbers, natural numbers (including zero), and positive integers, respectively.

Let $\Sigma$ be a  countable alphabet.
We write $\Sigma^*$ and $\Sigma^{\omega}$ to denote, respectively, the set of every finite and infinite word over $\Sigma$.
Moreover, $\Sigma^{\infty}$ denotes the set of finite and infinite words, i.e., $\Sigma^\infty\coloneqq\Sigma^* \cup \Sigma^{\omega}$. 
We use the convention that symbols with arrow on top will denote words, whereas symbols without arrow will denote alphabet elements.
Let $\seq{s} = s_1s_2\ldots$ be a word.
We write $\seq{s}_i$ to denote the $i$-th symbol $s_i$, and write $\seq{s}_{i\twodots j}$ to denote the subword $s_i\ldots s_j$, for $i<j$.
We use the convention that the indices of a word begin at $1$, so that the length of a word matches the index of the last symbol. 

Let $\seq{s}\in \Sigma^*$ and any $\seq{t}\in \Sigma^\infty$ be two words.
We denote the concatenation of $\seq{s}$ and $\seq{t}$ as $\seq{s}\seq{t}$.
We generalize this to sets of words: 
For $S\subseteq \Sigma^*$ and $T\subseteq \Sigma^\infty$, we define the concatenation $ST\coloneqq \set{\seq{s}\seq{t} \mid  \seq{s}\in S, \seq{t}\in T}$.
We say $\seq{s}$ is a prefix of $\seq{r}$, written $\seq{s}\pref \seq{r}$, if there exists a word $\seq{t}\in\Sigma^\infty$ such that $\seq{s}\seq{t} = \seq{r}$.

Suppose $\mathbb{T}\subseteq \mathbb{R}$ is a subset of real numbers, $v\in \mathbb{T}^n$ is a vector of length $n$ over $\mathbb{T}$, and $M\in \mathbb{T}^{n\times m}$ is a matrix of dimension $n\times m$ over $\mathbb{T}$; here $m,n$ can be infinity.
We use $v_i$ to denote the $i$-th element of $v$, and $M_{ij}$ to denote the element at the intersection of the $i$-th row and the $j$-th column of $M$.
A probability distribution over a set $S$ is a vector $v\in [0,1]^{|S|}$, such that $\sum_{i\in [1;|S|]} v_i=1$.

\subsection{Randomized Event Generators: Partially Observed Markov Chains}
We use partially observed Markov chains (POMC) as sequential randomized generators of events.
A POMC is a tuple $\tup{\Q,\trm,\initd,\Out,\OutFn}$, where
$\Q = \mathbb{N}^+$ is a countable set of states,
$\trm$ is a stochastic matrix of dimension $|\Q|\times |\Q|$, called the \emph{transition probability matrix},
$\initd$ is a probability distribution over $\Q$ representing the \emph{initial state distribution},
$\Out$ is a countable set of \emph{observations}, and
$\OutFn\colon \Q\to \Out$ is a function mapping every state to an \emph{observation}.
All POMCs in this paper are time-homogeneous, i.e., their transition probabilities do not vary over time. 

Semantically, every POMC $\mc$ induces a probability measure $\pr_\mc(\cdot)$ over the generated state and observation sequences.
For every finite state sequence $\seq{q} = q_1q_2\ldots q_t \in \Q^{*}$, the probability that $\seq{q}$ is generated by $\mc$ is given by $\pr_\mc(\seq{q}) = \initd_{q_1}\cdot\prod_{i=1}^{t-1} \trm_{q_iq_{i+1}}$.
Every finite state sequence $\seq{q}\in \Q^*$ for which $\pr_\mc(\seq{q})>0$ is called a finite \emph{internal path} of $\mc$; we omit $\mc$ if it is clear from the context.
The set of every internal path of length $n$ is denoted as $\paths[n]{\mc}$, and the set of every finite internal path is denoted as $\paths{\mc}$. 

Every finite internal path $\seq{q}$ can be extended to a set of infinite internal paths, which is called the cylinder set induced by $\seq{q}$, and is defined as $\cyl{\seq{q}}\coloneqq \set{\seq{r}\in \Q^\omega\mid \seq{q}\pref\seq{r}}$.
The probability measure $\pr_\mc(\cdot)$ on finite internal paths induces a pre-measure on the respective cylinder sets, which can be extended to a unique measure on the infinite internal paths by means of the Carath\'eodory's extension theorem \cite[pp.~757]{baier2008principles}.
The probability measure on the set of infinite internal paths is also denoted using $\pr_\mc(\cdot)$.

An external observer can only observe the observable part of an internal path of a POMC.
Given an internal path $\seq{q} = q_1q_2\ldots \in \Q^\infty$, we write $\OutFn(\seq{q})$ to denote the observation sequence $\OutFn(q_1)\OutFn(q_2)\ldots\in \Out^\infty$.
For a set of internal paths $S\subseteq \Q^\infty$, we write $\OutFn(S)$ to denote the respective set of observation sequences $\set{\seq{w}\in \Out^\infty\mid \exists \seq{q}\;.\;\seq{w}=\OutFn(\seq{q})}$.
An observation sequence $\seq{w}\in \Out^\infty$ is called an \emph{observed} path (of $\mc$) if there exists an internal path $\seq{q}$ for which $\OutFn(\seq{q})=\seq{w}$.
As before, we write $\outpaths[n]{\mc}$ for the set of every observed path of length $n$, and $\outpaths{\mc}$ for the set of every finite observed path.

We also use the inverse operator of $\OutFn$ to map every observed path $\seq{w}$ to the set of possible internal paths: $\OutFn^{-1}(\seq{w})\coloneqq\set{\seq{q}\in \Q^\infty\mid \OutFn(\seq{q}) = \seq{w}}$.
Furthermore, we extend $\OutFn^{-1}(\cdot)$ to operate over sets of observation sequences in the following way:
For any given $S\subseteq \Out^\infty$, define $\OutFn^{-1}(S)\coloneqq \set{\seq{q}\mid \exists \seq{w}\in S\;.\; \OutFn(\seq{q})=\seq{w}}$.

We abuse the notation and use $\pr_\mc(\cdot)$  to denote the induced probability measure on the set of observed paths, defined in the following way.
Given every set of finite observed paths $S\subseteq \Out^*$, we define $\pr_\mc(S)\coloneqq \sum_{\seq{q}\in\OutFn^{-1}(S)} \pr_\mc(\seq{q})$.
When the paths in a given set are infinite, the sum is replaced by integral.
We write $\seq{W}\sim \mc$ to denote the random variable that represents the distribution over finite sample observed paths generated by the POMC $\mc$.

\begin{example}
\label{ex:lending POMC}
As a running example, we introduce a POMC that models the sequential interaction between a bank and loan applicants.
Suppose there is a population of loan applicants, where each applicant has a credit score between $1$ and $4$, and belongs to either an advantaged group $A$ or a disadvantaged group $B$.
At every step, the bank receives loan application from one applicant, and, based on some unknown (but non-time-varying) criteria, decides whether to grant loan or reject the application.
We want to monitor, for example, the difference between loan acceptance probabilities for people belonging to the two groups.

The underlying POMC $\mc$ that models the sequence of loan application events is shown in Fig.~\ref{fig:illustrative POMC model of loan example}.
A possible internal path is $S(A,1)NS(A,4)YSB(A,3)N\ldots$, whose corresponding observed path is $SANSAYSBN\ldots$.
In our experiments, we use a more realistic model of the POMC with way more richer set of features for the individuals.
\end{example}

\begin{figure}
	\vspace{-2.5cm}
	\begin{tikzpicture}[scale=0.4,node distance=0.5cm,trim left=-6cm]
		\tikzstyle{arrow} = [thick,->,>=stealth]
		\tikzstyle{every node} = [font=\tiny,inner sep=0,outer sep=0]
		
		\node[state]	(s)		at	(0,0)	{$S$};
		\node[state]	(a1)		[below left=0.75 cm of s]	{$(A,4)$};
		\node[state]	(a2)		[left=of a1]		{$(A,3)$};
		\node[state]	(a3)		[left=of a2]		{$(A,2)$};
		\node[state]	(a4)		[left=of a3]		{$(A,1)$};
		\node[state]	(b1)		[below right=0.75 cm of s]	{$(B,1)$};
		\node[state]	(b2)		[right=of b1]		{$(B,2)$};
		\node[state]	(b3)		[right=of b2]		{$(B,3)$};
		\node[state]	(b4)		[right=of b3]		{$(B,4)$};
		\node[state]	(y)		[below =of a4]		{$Y$};
		\node[state]	(n)		[below =of b4]		{$N$};
		
		\begin{scope}[on background layer]
			\draw[fill=gray!15!white]	($(a4.south west)-(0.3,0.3)$)	rectangle	($(a1.north east)+(0.3,0.3)$);
			\draw[fill=gray!15!white]	($(b1.south west)-(0.3,0.3)$)	rectangle	($(b4.north east)+(0.3,0.3)$);
			\draw[fill=gray!15!white]	($(s.south west)-(0.3,0.3)$)		rectangle	($(s.north east)+(0.3,0.3)$);
			\draw[fill=gray!15!white]	($(y.south west)-(0.3,0.3)$)		rectangle	($(y.north east)+(0.3,0.3)$);
			\draw[fill=gray!15!white]	($(n.south west)-(0.3,0.3)$)		rectangle	($(n.north east)+(0.3,0.3)$);
		\end{scope}
		
		\node	at	($(a4.north)+(0,0.4)$)		{$A$};
		\node	at	($(b4.north)+(0,0.4)$)		{$B$};
		\node	at	($(s.north)+(0,0.4)$)		{$S$};
		\node	at	($(y.south)+(0,-0.4)$)		{$Y$};
		\node	at	($(n.south)+(0,-0.4)$)		{$N$};
		
		\path[->]
			(s)		edge		(a1)
					edge[bend right]		(a2)
					edge	[bend right]	(a3)
					edge	[bend right]	(a4)
					edge		(b1)
					edge[bend left]		(b2)
					edge[bend left]		(b3)
					edge[bend left]		(b4)
			(a1)		edge	[bend left]	(y)
					edge	[bend right]	(n)
			(a2)		edge	[bend left]	(y)
					edge	[bend right]	(n)
			(a3)		edge		(y)
					edge[bend right]		(n)
			(a4)		edge		(y)
					edge	[bend right]	(n)
			(b1)		edge	[bend left]	(y)
					edge	[bend right]	(n)
			(b2)		edge	[bend left]	(y)
					edge	[bend right]	(n)
			(b3)		edge	[bend left]	(y)
					edge[bend right]		(n)
			(b4)		edge	[bend left]	(y)
					edge		(n)
			(y)		edge	[out=155,in=125,looseness=1.3]	(s)
			(n)		edge	[out=25,in=55,looseness=1.3]	(s);
	\end{tikzpicture}
	\vspace{-1cm}
	\caption{The POMC modeling the sequential interaction between the bank and the loan applicants.
	The states $S$, $Y$, and $N$ respectively denote the start state, the event that the loan was granted (``$Y$'' stands for ``Yes''), and the event that the loan was rejected (``$N$'' stands for ``No'').
	Every middle state $(X,i)$, for $X\in \set{A,B}$ and $i\in \set{1,2,3,4}$, represents the group ($A$ or $B$) and the credit score $i$ of the current applicant.
	The states $S,Y,N$ are fully observable, i.e., their observation symbols coincide with their state symbols.
	The middle states are partially observable, with every $(A,i)$ being assigned the observation $A$ and every $(B,i)$ being assigned the observation $B$.
	The states with the same observation belong to the same shaded box.
	}
	\label{fig:illustrative POMC model of loan example}
\end{figure}
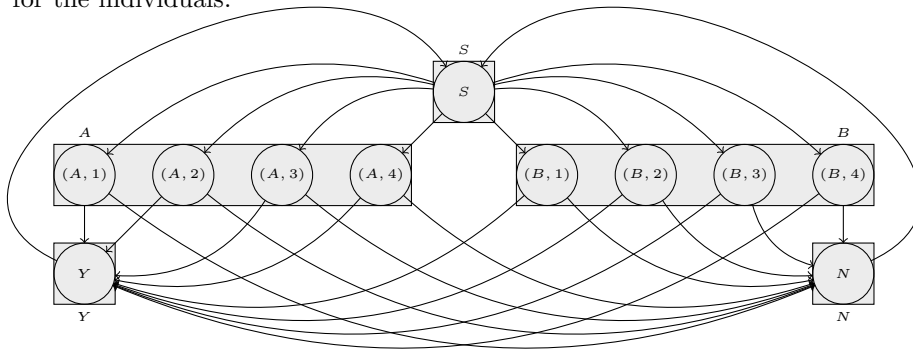
\vspace{-0.3cm}

\subsection{Register Monitors}
Our register monitors are adapted from the polynomial monitors of Ferr\`ere et al.\ \cite{ferrere2018theory}, and were also used in our previous work (in a more general randomized form) \cite{henzinger2023monitoring}.
Let $R$ be a finite set of integer variables called registers.
A function $v\colon R\to \mathbb{N}$ assigning concrete value to every register in $R$ is called a valuation of $R$.
Let $\mathbb{N}^R$ denote the set of all valuations of $R$.
Registers can be read and written according to relations in the signature $S=\langle 0,1,+,-,\times,\div,\leq \rangle$.
We consider two basic operations on registers:
\begin{itemize}[noitemsep,topsep=0pt]
	\item A \emph{test} is a conjunction of atomic formulas over $S$ and their negation;
	\item An \emph{update} is a mapping from variables to terms over $S$.
\end{itemize}
We use $\Phi(R)$ and $\Gamma(R)$ to respectively denote the set of tests and updates over $R$.
\emph{Counters} are special registers with a restricted signature $S=\langle 0,1,+,-,\leq \rangle$.

\begin{definition}[Register monitor]
	A register monitor is a tuple $\tup{\Sigma,\Lambda,R,v_{\mathsf{in}},f,\transrel}$ where 
	$\Sigma$ is a finite input alphabet, 
	$\verdict$ is an output alphabet, 
	$R$ is a finite set of registers,
	$v_{\mathsf{in}}\in \mathbb{N}^R$ is the initial valuation of the registers,
	$f\colon \mathbb{N}^R\to \verdict$ is an output function, and
	$\transrel\colon \Sigma\times\Phi(R)\to \Gamma(R)$ is the transition function such that
	for every $\sigma\in \Sigma$ and for every valuation $v\in \mathbb{N}^R$, there exists a unique $\phi\in \Phi(R)$ with $v\models\phi$ and $\transrel(\sigma,\phi)\in \Gamma(R)$.
\end{definition}

We refer to register monitors simply as monitors, and we fix the output alphabet $\Gamma$ as the set of every real interval.

A \emph{state} of a monitor $\monitor$ is a valuation of its registers $v\in \mathbb{N}^R$; the initial valuation $v_\init$ is the initial state.
The monitor $\monitor$ \emph{transitions} from state $v$ to another state $v'$ on input $\sigma\in \Sigma$ if there exists $\phi$ such that $v\models \phi$, there exists an update $\gamma=\transrel(\sigma,\phi)$, and if $v'$ maps every register $x$ to $v'(x)=v(\gamma(x))$.
The transition from $v$ to $v'$ on input $\sigma$ is written as $v\xrightarrow{\sigma} v'$.
A \emph{run} of $\monitor$ on a word $w_1\ldots w_t\in \Sigma^*$ is a sequence of transitions $v_1=v_\init\xrightarrow{w_1} v_2\xrightarrow{w_2}\ldots\xrightarrow{w_t} v_{t+1}$. 
The \emph{semantics} of the monitor is the function $\sem{\monitor}\colon \Sigma^*\to \Lambda$ that maps every finite input word to the last output of the monitor on the respective run.
For instance, the semantics of $\monitor$ on the word $\seq{w}$ is $\sem{\monitor}(\seq{w}) = f(v_{t+1})$.
An illustrative example of register monitors can be found in our earlier work \cite[Sec.~2.2]{henzinger2023monitoring}.

\section{Monitoring Quantitative Algorithmic Fairness Properties}

In our prior work on monitoring algorithmic fairness for \emph{fully} observable Markov chains \cite{henzinger2023monitoring}, we formalized (quantitative) algorithmic fairness properties using the so-called Probabilistic Specification Expressions (PSE).
A PSE $\varphi$ is an arithmetic expression over the variables of the form $v_{ij}$, for $i,j\in Q$ for a finite set $Q$.
The semantics of $\varphi$ is interpreted statically over a given Markov chain $M$ with state space $Q$, by replacing every $v_{ij}$ with the transition probability from the state $i$ to the state $j$ in $M$.
The algorithmic question we considered is that given a PSE $\varphi$, how to construct a monitor that will observe one long path of an unknown Markov chain, and after each observation will output a PAC estimate of the value of $\varphi$ with a pre-specified confidence level.

An exact representation of the above problem formulation is not obvious for POMCs.
In particular, while it is reasonable to generalize the semantics of PSEs to be over the probabilities between \emph{observations} instead of probabilities between states, it is unclear how these probabilities will be defined.
In the following, we use simple examples to illustrate several cruxes of formalizing algorithmic fairness on POMCs, and motivate the use of the assumptions of stationary distribution, irreducibility, and positive recurrence (formally stated in Assump.~\ref{ass:stationarity}) to mitigate the difficulties.
These assumptions will later be used to formalize the algorithmic fairness properties in Sec.~\ref{sec:BSE}.

In the following, we will write $\st$ to denote the \emph{stationary distribution} of Markov chains with transition matrix $\trm$, i.e., $\pi = \trm \pi$.

\subsection{Role of the Stationary Distribution}

First, we demonstrate in the following example that POMCs made up of unfair sub-components may have overall fair behavior in the stationary distribution, which does not happen for fully observable Markov chains.

\begin{example}\label{ex:two coins:uniform switch}
Suppose there are two coins $A$ and $B$, where $A$ comes up with heads with probability $0.9$ and $B$ comes up with tails with probability $0.9$. 
We observe a sequence of coin tosses (i.e., the observations are heads and tails), without knowing which of the two coins (the state) was tossed.
If the choice of the coin at each step is made uniformly at random, then, intuitively, the system will produce fair outcomes in the long run, with equal proportions of heads and tails being observed in expectation.
Thus, although each coin was unfair, we can still observe overall fair outcome, provided the fraction of times each coin was chosen in the stationary distribution balances out the unfairness in the coins themselves.

To make the above situation more concrete, imagine that the underlying POMC has two states $a,b$ (e.g., $a,b$ represent the states when $A,B$ are selected for tossing, respectively) with the same observation (which coin is selected is unknown to the observer), where the measures of the given fairness condition (e.g., the biases of the coins $A,B$) are given by $f_a, f_b$.
We argue that, intuitively, the overall fairness of the POMC is given by $\st_af_a+\st_bf_b$.
This type of analysis is unique to POMCs, whereas for fully observable Markov chains, computation of fairness is simpler and can be done without involving the stationary distribution.
\end{example}

In the next example, we demonstrate some challenges of monitoring fairness when we express fairness by weighing in the stationary distribution as above.

\begin{example}\label{ex:two coins:one way switch}
Consider the setting of Ex.~\ref{ex:two coins:uniform switch}, and suppose now only the initial selection of the coin happens uniformly at random but subsequently the same coin is used forever.
If we consider the underlying POMC, both $\pi_a,\pi_b$ will be $0.5$, because the initial selection of the coin happens uniformly at random.
However, the monitor will observe the toss outcomes of only one of the two coins on a given trace.
It is unclear how the monitor can extrapolate its estimate to the overall fairness property $\st_af_a+\st_bf_b$ in this case.
\end{example}

To deal with the situations described in Ex.~\ref{ex:two coins:uniform switch} and Ex.~\ref{ex:two coins:one way switch}, we will make the following assumption.

\begin{assumption}\label{ass:stationarity}
	We assume that the POMCs are irreducible, positively recurrent, and are initialized in their stationary distributions.
\end{assumption}

The irreducibility and positive recurrence guarantees existence of the stationary distribution.
Assump.~\ref{ass:stationarity} ensures that, firstly, we will see every state infinitely many times (ruling out the above corner-case), and, secondly, the proportion of times the POMC will spend in all the states will be the same (given by the stationary distribution) all the time.
While Assump.~\ref{ass:stationarity} makes it easier to formulate and analyze the algorithmic fairness properties over POMCs, monitoring these properties over POMCs still remains a challenging problem due to the non-Markovian nature of the observed path.

\subsection{Bounded Specification Expressions}\label{sec:BSE}
We introduce bounded specification expressions (\vbse) to formalize the fairness properties that we want to monitor.
A \vbse assigns values to finite word patterns of a given alphabet.
The main components of a \vbse are \emph{atomic functions}, where an atomic function $f_n$ assigns bounded real values to observation sequences of length $n$, for a given $n\in \mathbb{N}^+$.
An atomic function $f_n$ can express quantitative star-free regular expressions, assigning real values to words of length $n$.

Following are some examples.
Let $\Out = \set{r,g}$ be an observation alphabet, where $r$ stands for ``request'' and $g$ stands for ``grant.''
A boolean atomic function $f_{2}$, with $f_2(rr)=0$ and $f_2(\seq{w})=1$ for every $\seq{w}\in \Out^{2}\setminus \set{rr}$, can express the property that two requests should not appear consecutively.
An integer-valued atomic function $f_{10}$, with $f_{10}(rr^{i}g\seq{w})=i$ when $i\in [0;8]$ and $\seq{w}\in\Out^{8-i}$, and with $f_{10}(\seq{z})=8$ when $\seq{z}\in \Out^{10}\setminus rr^{i}g\Out^{8-i}$, assigns to any sub-sequence the total waiting time between a request and the subsequent grant, while saturating the waiting time to $8$ when it is above $8$.
The specified word-length $n$ for any atomic function $f_n$ is called the \emph{arity} of $f_n$.
Let $P$ be the set of all atomic functions over a given observation alphabet.

A \vbse may also contain arithmetic and/or logical connectives and relational operators to express complex value-based properties of an underlying probabilistic generator, like the POMCs.
We consider two fragments of \vbse-s, expressing qualitative and quantitative properties, and called, respectively, \qualvbse and \quanvbse in short.
The syntaxes of the two types of \vbse-s are given as:
\begin{subequations}\label{equ:syntax property}
	\begin{alignat}{2}
		\text{(\quanvbse)}\qquad\varphi &\Coloneqq \kappa \in \mathbb{R} \ | \ f\in P \ | \  \varphi + \varphi  \ | \ \varphi\cdot\varphi  \ | \ 1\div\varphi  \ | \ (\varphi), \label{equ:syntax probability prop.}\\
		\text{(\qualvbse)}\qquad\psi &\Coloneqq \true \ | \ \varphi \geq 0 \ | \ \lnot \psi \ | \ \psi \land \psi.
	\end{alignat}
	\end{subequations}

The semantics of a \quanvbse $\varphi$ over the alphabet $\Sigma$ is interpreted over POMCs satisfying Assump.~\ref{ass:stationarity} and with observations $\Sigma$. 
When $\varphi$ is an atomic function $f\colon \Sigma^{n}\to [a,b]$ for some $n \in \mathbb{N}^+$, $a,b\in \mathbb{R}$, then, for a given POMC $\mc$, the semantics of $\varphi$ is defined as follows.
For every time $t\in \mathbb{N}^+$,
\begin{align}
	\varphi(\mc) = f(\mc) \coloneqq \int_{\Out^\omega} f(\seq{\out}_{t:t+n-1})d\pr_\mc(\seq{\out}).
\end{align}
The definition of $f(\mc)$ is well-defined, because $f(\mc)$ will be the same for every $t$, since the POMC will remain in the stationary distribution forever (by Assump.~\ref{ass:stationarity} and by the property of stationary distributions).
Intuitively, the semantics $f(\mc)$ represents the expected value of the function $f$ on any sub-word of length $n$ on any observed path of the POMC, when it is known that the POMC is in the stationary distribution (Assump.~\ref{ass:stationarity}).

The arithmetic operators in \quanvbse-s have the usual semantics (``$+$'' for addition, ``$-$'' for difference, ``$\cdot$'' for multiplication, and ``$\div$'' for division).

On the other hand, the semantics of a \qualvbse $\psi$ is boolean, which inductively uses the semantics of the constituent $\varphi$ expressions.
For a \qualvbse $\psi=\varphi\geq 0$, the semantics of $\psi$ is given by:
\begin{align*}
	\psi(\mc) \coloneqq \begin{cases}
					\true	&	\text{if }\varphi(\mc)\geq 0,\\
					 \false	&	\text{otherwise}.
				\end{cases}
\end{align*}
The semantics of the boolean operators in $\psi$ is the usual semantics of boolean operators in propositional logic.
The following can be added as syntactic sugar:
``$\varphi \geq c$'' for a constant $c$ denotes ``$\varphi'\geq 0$'' with $\varphi'\coloneqq \varphi -c$,
``$\varphi \leq c$'' denotes ``$-\varphi \geq -c$,'' 
``$\varphi = c$'' denotes ``$(\varphi\geq c)\land (\varphi\leq c)$,''
``$\varphi > c$'' denotes ``$\lnot(\varphi \leq c)$,''
``$\varphi < c$'' denotes ``$\lnot(\varphi \geq c)$,'' and
``$\psi\lor\psi$'' denotes ``$\lnot(\lnot\psi \land \lnot\psi)$.''

\medskip
\noindent\textbf{Fragment of \vbse: Probabilistic Specification Expressions (PSEs):}
In our prior work \cite{henzinger2023monitoring}, we introduced PSEs to model algorithmic fairness properties of Markov chains with fully observable state space.
PSEs are arithmetic expressions over atomic variables of the form $v_{ij}$, where $i,j$ are the states of the given Markov chain, and whose semantic value equals the transition probability from $i$ to $j$.
The semantics of a PSE is then the valuation of the expression obtained by plugging in the respective transition probabilities.
We can express PSEs using \quanvbse-s as below.
For every variable $v_{ij}$ appearing in a given PSE, we use the atomic function $f$ that assigns to every finite word $\seq{w}\in\Sigma^*$ the ratio of the number of $(i,j)$ transitions to the number of occurrences of $i$ in $\seq{w}$.
We will denote the function $f$ as $P(j\mid i)$ in this case, and, in general, $i,j$ can be observation labels for the case of \quanvbse-s.
It is straightforward to show that semantically the two expressions will be the same.
On the other hand, \quanvbse-s are strictly more expressive than PSEs.
For instance, unlike PSEs, \quanvbse-s can specify probability of transitioning from one observation label to another, the average number of times a given state is visited on any finite path of a Markov chain, etc.

\medskip
\noindent\textbf{Fragment of \vbse: Probabilities of Sequences:}
We consider a useful fragment that expresses the probability that a sequence from a given set $S\subseteq \Sigma^*$ of finite observation sequences will be observed at any point in time on any observed path.
We assume that the length of every sequence in $S$ is uniformly bounded by some integer $n$.
Let $\overline{S}\subseteq \Sigma^n$ denote the set of extensions of sequences in $S$ up to length $n$, i.e., $\overline{S}\coloneqq\set{\seq{w}\in\Sigma^n\mid \exists \seq{u}\in S\;.\;\seq{u}\pref\seq{w}}$.
Then the desired property will be expressed simply using an atomic function with $f\colon \Sigma^n\to \set{0,1}$ being the indicator function of the set $\overline{S}$, i.e., $f(\seq{w})=1$ iff $\seq{w}\in \overline{S}$.
It is straightforward to show that, for a given POMC $\mc$, the semantics $f(\mc)$ expresses the desired property. 
For a set of finite words $S\subseteq \Sigma^*$, we introduce the shorthand notation $P(S)$ to denote the probability of seeing an observation from the set $S$ at any given point in time.
Furthermore, for a pair of sets of finite words $S,T\subseteq \Sigma^*$, we use the shorthand notation $P(S\mid T)$ to denote $\sfrac{P(TS)}{P(T)}$, which represents the conditional probability of seeing a word in $S$ after we have seen a word in $T$.

\begin{example}[Group fairness.]
\label{ex:fairness properties in lending POMC}
	Consider the setting in Ex.~\ref{ex:lending POMC}.
	We show how we can represent various group fairness properties using \quanvbse-s.
	Demographic parity \cite{dwork2012fairness} quantifies bias as the difference between the probabilities of individuals from the two demographic groups getting the loan, which can be expressed as $P(Y\mid A)-P(Y\mid B)$.
	Disparate impact \cite{feldman2015certifying} quantifies bias as the ratio between the probabilities of getting the loan across the two demographic groups, which can be expressed as $P(Y\mid A)\div P(Y\mid B)$.
\end{example}

In prior works \cite{albarghouthi2019fairness,henzinger2023monitoring}, group fairness properties could be expressed on strictly less richer class of fully observed Markov chain models, where the features of each individual were required to contain only their group information.
An extension to the model of Ex.~\ref{ex:lending POMC} is not straightforward as the confidence interval used in these works would not be applicable.

\begin{example}[Social fairness.]\label{ex:social fairness in lending}
	Consider the setting in Ex.~\ref{ex:lending POMC}, except that now the credit score of each individual will be observable along with their group memberships, i.e., each observation is a pair of the form $(X,i)$ with $X\in \set{A,B}$ and $i\in\set{1,2,3,4}$.
	There may be other non-sensitive features, such as age, which may be hidden.
	We use the social fairness property  \cite{henzinger2023runtime} quantified as the difference between the expected credit scores of the groups $A$ and $B$.
	To express this property, we use the unary atomic functions $f_1^X\colon \Sigma\to \mathbb{N}$, for $X\in \set{A,B}$, such that $f_1^X\colon (Y,i)\mapsto i$ if $Y=X$ and is $0$ otherwise.
	The semantics of $f_1^X$ is the expected credit score of group $X$ scaled by the probability of seeing an individual from group $X$.
	Then social fairness is given by the \quanvbse $\varphi = \frac{f_1^A}{P(A)} - \frac{f_1^B}{P(B)}$.
\end{example}

\begin{example}[Quantitative group fairness.]\label{ex:fair salary distribution in hiring}
	Consider a sequential hiring scenario where at each step the salary and a sensitive feature (like gender) of a new recruit are observed.
	We denote the pair of observations as $(X,i)$, where $X\in \set{A,B}$ represents the group information based on the sensitive feature and $i$ represents the salary.
	We can express the disparity in expected salary of the two groups in a similar manner as in Ex.~\ref{ex:social fairness in lending}.
	Define the unary functions $f_1^X\colon \Sigma\to \mathbb{N}$, for $X\in \set{A,B}$, such that $f_1^X\colon (Y,i)\mapsto i$ if $Y=X$ and is $0$ otherwise.
	The semantics of $f_1^X$ is the expected salary of group $X$ scaled by the probability of seeing an individual from group $X$.
	Then the group fairness property is given by the \quanvbse $\varphi = \frac{f_1^A}{P(A)} - \frac{f_1^B}{P(B)}$.
\end{example}

\subsection{Problem Statement}
Informally, our goal is to build monitors that will observe randomly generated observed paths of increasing length from  a given unknown POMC, and, after each observation, will generate an updated estimate of how fair or biased the system was until the current time.
Since the monitor's estimate is based on statistics collected from a finite path, the output may be incorrect with some probability.
That is, the source of randomness is from the fact that the prefix is a finite sample of the fixed but unknown POMC.

For a given $\delta\in (0,1)$, and a given \vbse $\varphi$, we define a \emph{problem instance} as the tuple $\tup{ \varphi,\delta}$.

\begin{prob}[Monitoring \quanvbse-s]\label{prob:frequentist:quantitative}
	Suppose $\tup{\varphi,\delta}$ is a problem instance where $\varphi$ is a \quanvbse.
	Design a monitor $\monitor$, with output alphabet $\set{[l,u]\mid l,u\in\mathbb{R}\;.\; l < u}$, such that for every POMC $\mc$ satisfying Assump.~\ref{ass:stationarity}, we have:
	\begin{align}\label{equ:problem:frequentist:quantitative}
		\pr_{\seq{\rout}\sim\mc}\left(\varphi(\mc)\in \sem{\monitor}(\seq{\rout})\right) \geq 1-\delta.
	\end{align}
\end{prob}

The estimate $[l,u]=\sem{\monitor}(\seq{\out})$ is called the $(1-\delta)\cdot 100\%$ \emph{confidence interval} for $\varphi(M)$.
The radius, given by $\varepsilon=0.5\cdot (u-l)$, is called the \emph{estimation error}, the quantity $\delta$ is called the \emph{failure probability}, and the quantity $1-\delta$ is called the \emph{confidence}.
Intuitively, the monitor outputs the estimated confidence interval that contains the range of values within which the true semantic value of $\varphi$ falls with $(1-\delta)\cdot 100\%$ probability.
The estimate gets more precise as the error gets smaller, and the confidence gets higher.
We will prefer the monitor with the maximum possible precision, i.e., having the least estimation error for a given $\delta$.

\begin{prob}[Monitoring \qualvbse-s]\label{prob:frequentist:qualitative}
	Suppose $\tup{\varphi,\delta}$ is a problem instance where $\varphi$ is a \qualvbse.
	Design a monitor $\monitor$, with output alphabet $\set{\true,\false}$, such that for every POMC $\mc$ satisfying Assump.~\ref{ass:stationarity}, we have:
	\begin{align}\label{equ:problem:frequentist:qualitative}
		&\pr_{\seq{\rout}\sim\mc}\left( \psi(\mc) \mid \sem{\monitor}(\seq{\rout}) = \true\right) \geq 1- \delta,\\
		&\pr_{\seq{\rout}\sim\mc}\left( \lnot\psi(\mc) \mid \sem{\monitor}(\seq{\rout}) = \false\right) \geq 1- \delta.
	\end{align}
\end{prob}

Unlike Prob.~\ref{prob:frequentist:quantitative}, the monitors addressing Prob.~\ref{prob:frequentist:qualitative} do not output an interval but output a boolean verdict.
Intuitively, the output of the monitor for Prob.~\ref{prob:frequentist:qualitative} is either $\true$ or $\false$, and it is required that the semantic value of the property $\psi$ is, respectively, $\true$ or $\false$ with $(1-\delta)\cdot 100\%$ probability.

\section{Construction of the Monitor}

\vspace{-0.2cm}

Our overall approach in this work is similar to the prior works \cite{albarghouthi2019fairness,henzinger2023monitoring,henzinger2023runtime}:
We first compute a point estimate of the given \vbse from the finite observation sequence of the POMC, and then compute an interval estimate through known concentration inequalities.
However, the same concentration inequalities as the prior works cannot be applied, because they required two successive observed events be independent, which is not true for POMCs.
For instance, in Ex.~\ref{ex:two coins:one way switch}, if we start the sequence of tosses by first tossing coin $A$, then we know that the subsequent tosses are going to be done using $A$ only, thereby implying that the outcomes of the future tosses will be statistically dependent on the initial random process that chooses between the two coins at the first step.

We present a novel theory of monitors for \vbse-s on POMCs satisfying Assump.~\ref{ass:stationarity}, using McDiarmid-style concentration inequalities for hidden Markov chains.
In Sec.~\ref{sec:point estimator for atoms} and \ref{sec:interval estimator for atoms}, we first present, respectively, the point estimator and the monitor for an individual atom.
In Sec.~\ref{sec:the complete monitor}, we build the overall monitor by combining the interval estimates of the individual atoms through interval arithmetic and union bound.

\subsection{A Point Estimator for the Atoms}
\label{sec:point estimator for atoms}

 Consider a \vbse atom $f$.
 We present a point estimator for $f$, which computes an estimated value of $f(\mc)$ from a finite observed path $\seq{\out}\in \Sigma^t$, of an arbitrary length $t$, of the unknown POMC $\mc$.
 The point estimator $\hat{f}(\cdot)$ is given as:
\begin{align}
	\hat{f}(\seq{\out})\coloneqq \frac{1}{t-n+1} \sum_{i=1}^{t-n+1} f(\seq{\out}_{i\twodots i+n-1}). \label{equ:point estimator}
\end{align}

 In the following proposition, we establish the unbiasedness of the estimator $\hat{f}(\cdot)$, a desirable property that says that the expected value of the estimator's output will coincide with the true value of the property that is being estimated.

\begin{proposition}
	\label{lemma:expectation atom}
	Let $\mc$ be a POMC satisfying Assump.~\ref{ass:stationarity}, $f\colon\Out^{n}\to[a,b]$ be a function for fixed $n$, $a$, and $b$, and $\seq{\rout} \sim \mc$ be a random observed path of an arbitrary length $|\seq{\rout}|=t > n$.
	Then $\expe(\hat{f}(\seq{\rout}))=f(\mc)$.
\end{proposition}

The following corollary establishes the counterpart of Prop.~\ref{lemma:expectation atom} for the fragment of \vbse with probabilities of sequences.

\begin{corollary}
	\label{lemma:probability atom}
	Let $\mc$ be a POMC satisfying Assump.~\ref{ass:stationarity}, $\specset\subset \Out^*$ be a set of bounded length observation sequences with bound $n$, $f:\Out^n\to \set{0,1}$ be the indicator function of the set $\overline{\specset}$, and $\seq{\rout} \sim \mc$ be a random observed path of an arbitrary length $|\seq{\rout}|>n$.
	Then $\expe(\hat{f}(\seq{\rout}))=P(\specset)$.
\end{corollary}

\subsection{The Atomic Monitor}
\label{sec:interval estimator for atoms}

\begin{algorithm}
 	\caption{$\mathit{Monitor}_{(f,\delta)}$: Monitor for $(f,\delta)$ where $f\colon \Sigma^n\to [a,b]$ is an atomic function of a \vbse}
 	\label{alg:atomic monitor}
 		\begin{minipage}{0.34\textwidth}
 			\begin{algorithmic}[1] 
 			\Function{$\mathit{Init}()$}{}
 				\State $t\gets 0$ \Comment{current time}
 				\State $y\gets 0$ \Comment{current point estimate}
 				\State $\seq{w}\gets \underbrace{\bot \ldots \bot}_{n \text{ times}}$ \Comment{a dummy word of length $n$, where $\bot$ is the dummy symbol}
 			\EndFunction
 			\end{algorithmic}
 		\end{minipage}
 		\begin{minipage}{0.65\textwidth}
 			\begin{algorithmic}[1]
 			\Function{$\mathit{Next}(\sigma)$}{}
 				\State $t \gets t+1$ \Comment{progress time}
 				\If{$t<n$} \Comment{too short observation sequence}
 					\State $\seq{w}_t \gets \sigma$ 
 					\State \Return $\bot$ \Comment{inconclusive}
 				\Else
 					\State $\seq{w}_{1\twodots n-1} \gets \seq{w}_{2\twodots n}$ \Comment{shift window}\label{step:shift window}
 					\State $\seq{w}_n \gets \sigma$ \Comment{add the new observation}
 					\State $x \gets f(\seq{w})$ \Comment{latest evaluation of $f$}
 					\State $y \gets \left(y*(t-n)+x\right)/(t-n+1)$ \Comment{running av.\ impl.\ of Eq.~\ref{equ:point estimator}}
 					\State $\varepsilon\gets \sqrt{-\ln(\delta/2)\cdot \frac{t\cdot \min(t-n+1,n) \cdot 9 \cdot \tau_{mix}}{2 (t-n+1)^2}}$ 
					\ifarxiv 					
 					\Comment{ PAC bound, see Thm.~\ref{thm:mcdiarmid} in the appendix}
 					\fi
 					\State \Return $[y-\varepsilon,y+\varepsilon]$  \Comment{confidence interval}
 				\EndIf
 			\EndFunction
 		\end{algorithmic}
 		\end{minipage}
 \end{algorithm}

A monitor for each individual atom is called an atomic monitor, which serves as the building block for the overall monitor.
Each atomic monitor is constructed by computing an interval estimate of the semantic value $f(\mc)$ for the respective atom $f$ on the unknown POMC $\mc$.
For computing the interval estimate, we use the McDiarmid-style inequality 
\ifarxiv
(see Thm.~\ref{thm:mcdiarmid} in appendix) 
\fi
to find a bound on the width of the interval around the point estimate $\hat{f}(\cdot)$.

McDiarmid's inequality is a concentration inequality bounding the distance between the sample value and the expected value of a function satisfying the bounded difference property when evaluated on independent random variables. 
There are several works extending this result to functions evaluated over a sequence of dependent random variables, including Markov chains \cite{paulin2015concentration,esposito2023concentration,kontorovich2017concentration}. 
In order to use McDiarmid's inequality, we will need the following standard \cite{levin2017markov} additional assumption on the underlying POMC.
\begin{assumption}\label{assump:aperiodicity}
	We assume that the POMCs are aperiodic, and that the mixing time of the POMC is bounded by a known constant $\taumix$.
\end{assumption}

We summarize the algorithmic computation of the atomic monitor in Alg.~\ref{alg:atomic monitor}, and establish its correctness in the following theorem.

\begin{theorem}[Solution of Prob.~\ref{prob:frequentist:quantitative} for atomic formulas]
\label{thm:soundness of atomic monitor}
	Let $\tup{f,\delta}$ be a problem instance where $f\colon \Out^n\to [a,b]$ is an atomic formula for some fixed $n$, $a$, and $b$. 
	Moreover, suppose the given unknown POMC satisfies Assump.~\ref{assump:aperiodicity}.
	Then Algorithm \ref{alg:atomic monitor} implements a monitor solving Problem \ref{prob:frequentist:quantitative} for the given problem instance.
	The monitor requires $\mathcal{O}(n)$-space, and, after arrival of each new observation, computes the updated output in $\mathcal{O}(n)$-time.
\end{theorem}

The confidence intervals generated by McDiarmid-style inequalities for Markov chains tighten in relation to the mixing time of the Markov chain. 
This means the slower a POMC mixes, the longer the monitor needs to watch to be able to obtain an output interval of the same quality.

\subsection{The Complete Monitor}
\label{sec:the complete monitor}

The final monitors for \quanvbse-s and \qualvbse-s are presented in Alg.~\ref{alg:quant monitor} and Alg.~\ref{alg:qual monitor}, respectively, where we recursively combine the interval estimates of the constituent sub-expressions using interval arithmetic and the union bound.
Similar idea was used by Albarghouthi et al.~\cite{albarghouthi2019fairness}.
The correctness and computational complexities of the monitors are formally stated below.

\begin{theorem}[Solution of Prob.~\ref{prob:frequentist:quantitative}]
	Let $\tup{\varphi_1\odot\varphi_2,\delta_1+\delta_2}$ be a problem instance where $\varphi_1,\varphi_2$ are a pair of \quanvbse-s and $\odot\in\set{+,\cdot,\div}$.
	Moreover, suppose the given unknown POMC satisfies Assump.~\ref{assump:aperiodicity}.
	Then Alg.~\ref{alg:quant monitor} implements the monitor $\monitor$ solving Problem~\ref{prob:frequentist:quantitative} for the given problem instance.
	If the total number of atoms in $\varphi_1\odot\varphi_2$ is $k$ and if the arity of the largest atom in $\varphi_1\odot\varphi_2$ is $n$, then $\monitor$ requires $\mathcal{O}(k+n)$-space, and, after arrival of each new observation, computes the updated output in $\mathcal{O}(k\cdot n)$-time.
\end{theorem}

\begin{theorem}[Solution of Prob.~\ref{prob:frequentist:qualitative}]
	Let $\tup{\psi,\delta}$ be a problem instance where $\psi$ is a \qualvbse.
	Moreover, suppose the given unknown POMC satisfies Assump.~\ref{assump:aperiodicity}.
	Then Alg.~\ref{alg:qual monitor} implements the monitor $\monitor$ solving Problem~\ref{prob:frequentist:qualitative} for the given problem instance.
	If the total number of atoms in $\psi$ is $k$ and if the arity of the largest atom in $\psi$ is $n$, then $\monitor$ requires $\mathcal{O}(k+n)$-space, and, after arrival of each new observation, computes the updated output in $\mathcal{O}(k\cdot n)$-time.
\end{theorem}

	\begin{algorithm}
		\caption{$\mathit{Monitor}_{(\psi,\delta)}$}
		\label{alg:qual monitor}
		\begin{minipage}{0.47\textwidth}
		\begin{algorithmic}[1]
			\Function{$\mathit{Init}()$}{}
				\If{$\psi\equiv \varphi \geq 0$}
					\State $\aut\gets \mathit{Monitor}_{(\varphi,\delta)}$
					\State $\aut.\mathit{Init}()$
				\ElsIf{$\psi \equiv \lnot \psi_1$}
					\State $\aut \gets \mathit{Monitor}_{(\psi_1,\delta)}$
					\State $\aut.\mathit{Init}()$
				\ElsIf{$\psi\equiv \psi_1\land \psi_2$}
					\State Choose $\delta_1,\delta_2$ s.t.\ $\delta = \delta_1+\delta_2$
					\State $\aut_1\gets \mathit{Monitor}_{(\psi_1,\delta_1)}$
					\State $\aut_2\gets \mathit{Monitor}_{(\psi_2,\delta_2)}$
					\State $\aut_1.\mathit{Init}()$
					\State $\aut_2.\mathit{Init}()$
				\EndIf
			\EndFunction
		\end{algorithmic}
		\end{minipage}
		\begin{minipage}{0.52\textwidth}
		\begin{algorithmic}[1]
			\Function{$\mathit{Next}$}{$\sigma$}
				\If{$\psi\equiv \varphi \geq 0$}
					\State $[l,u]\gets\aut.\mathit{Next}(\sigma)$
					\If{$l\geq 0$}
						\Return $\mathit{true}$
					\ElsIf{$u\leq 0$}
						\Return $\mathit{false}$
					\Else\
						\Return $\bot$ \Comment{don't know, we assume $\lnot\bot = \bot\land \true = \bot\land\false=\bot$.}
					\EndIf
				\ElsIf{$\psi \equiv \lnot \psi_1$}
					\State \Return $\lnot \left(\aut.\mathit{Next}(\sigma)\right)$
				\ElsIf{$\psi\equiv \psi_1\land \psi_2$}
					\State \Return $\aut_1.\mathit{Next}(\sigma)\land\aut_2.\mathit{Next}(\sigma)$
				\EndIf
			\EndFunction
		\end{algorithmic}
		\end{minipage}
	\end{algorithm}

\section{Experiments}

We implemented our monitoring algorithm in Python, and applied it to the real-world lending example \cite{damour2020fairness} described in Ex.~\ref{ex:lending POMC} and to an academic example called hypercube. 
We ran the experiments on a MacBook Pro (2023) with Apple M2 Pro processor and 16GB of RAM.

\smallskip
\noindent\textbf{The Lending Example.}
The underlying POMC model (unknown to the monitor) of the system is approximately as shown in Fig.~\ref{fig:illustrative POMC model of loan example} with a few differences.
Firstly, we added a low-probability self-loop on the state $S$ to ensure aperiodicity.
Secondly, we considered only two credit score levels.
\begin{wrapfigure}{r}{0.46\textwidth}
\begin{minipage}{0.46\textwidth}
	\vspace{-1.5cm}
	\begin{algorithm}[H]
		\caption{$\mathit{Monitor}_{(\varphi_1\odot\varphi_2,\delta_1+\delta_2)}$}
		\label{alg:quant monitor}
		\begin{algorithmic}[1]
			\Function{$\mathit{Init}()$}{}
				\State $\aut_1\gets \mathit{Monitor}_{(\varphi_1,\delta_1)}$
				\State $\aut_2\gets \mathit{Monitor}_{(\varphi_2,\delta_2)}$
				\State $\aut_1.\mathit{Init}()$
				\State $\aut_2.\mathit{Init}()$
			\EndFunction
		\end{algorithmic}
		\begin{algorithmic}[1]
			\Function{$\mathit{Next}$}{$\sigma$}
				\State $[l_1,u_1]\gets \aut_1.\mathit{Next}(\sigma)$
				\State $[l_2,u_2]\gets \aut_2.\mathit{Next}(\sigma)$
				\State \Return $[l_1,u_1]\odot [l_2,u_2]$ \Comment{interval arithmetic}
			\EndFunction
		\end{algorithmic}
	\end{algorithm}
	\vspace{-1.5cm}
\end{minipage}
\end{wrapfigure}
Thirdly, there are more hidden states (in total $171$ states) in the system, like the action of the individual (repaying the loan or defaulting), etc.
We monitor demographic parity, defined as $\varphi_{\mathsf{DP}} \coloneqq P(Y \mid A)-P(Y\mid B)$, and an absolute version of it, defined as $\varphi_{\mathsf{TDP}} \coloneqq P(AY) - P(BY)$.
While $\varphi_{\mathsf{DP}}$ represents the difference in probabilities of giving loans to individuals from the two groups ($A$ and $B$), $\varphi_{\mathsf{TDP}}$ represents the difference in joint probabilities of selecting and then giving loans to individuals from the two groups. 
None of the two properties can be expressed using the previous formalism \cite{albarghouthi2019fairness,henzinger2023monitoring}, because $\varphi_{\mathsf{DP}}$ requires conditioning on observations, and $\varphi_{\mathsf{TDP}}$ requires expressing absolute probabilities, which were not considered before.

After receiving new observations, the monitors for $\varphi_{\mathsf{DP}}$ and $\varphi_{\mathsf{TDP}}$ took, respectively, $47\,\si{\microsecond}$ and $18\,\si{\microsecond}$ on an average (overall $43\,\si{\microsecond}$--$0.2\,\si{\second}$ and $12\,\si{\microsecond}$--$3.2\,\si{\second}$) to update their outputs, showing that our monitors are fast in practice.

Fig.~\ref{fig:lending} shows the outputs of the monitors for $\delta=0.05$ (i.e., $95\%$ confidence interval).
For the POMC of the lending example, we used a pessimistic bound $\taumix=170589.78$~steps on the mixing time (computation as in \cite{jerison2013general}), with which the estimation error $\varepsilon$ shrinks rather slowly in both cases. 
For example, for $\varphi_{\mathsf{TDP}}$, in order to get from the trivial value $\varepsilon=1$ (the confidence interval spans the entire range of possible values) down to $\varepsilon=0.1$, the monitor requires about $4\cdot 10^9$ observations.
For $\varphi_{\mathsf{DP}}$, the monitor requires even more number of observations ($\sim 10^{12}$) to reach the same error level.
This is because $\varphi_{\mathsf{DP}}$ involves conditional probabilities requiring divisions, which amplify the error when composed using interval arithmetics.
We conclude that a direct division-free estimation of the conditional probabilities, together with tighter bounds on the mixing time will significantly improve the long-run accuracy of the monitor.

\smallskip
\noindent\textbf{The Hypercube Example.}
We considered a second example \cite[pp.\ 63]{levin2017markov}, whose purpose is to demonstrate that the tightness of our monitors' outputs is sensitive to the choice of the bound on the mixing time.
The POMC models a random walk along the edges of a hypercube $\{0,1\}^n$, where each vertex of the hypercube represents a state in the POMC and states starting with $0$ and $1$ are mapped to the observations $a$ and $b$, respectively.
We fix $n$ to $3$ in our experiments.
At every step, the current vertex is chosen with probability $\sfrac{1}{2}$, and every neighbor is chosen with probability $\sfrac{1}{2n}$.
A tight bound on the mixing time of this POMC is given by $\tau_{\mathsf{true\,mix}} =n(\log{n} + \log{4})$~steps \cite[pp.\ 63]{levin2017markov}.
We consider the properties $\psi_{\mathsf{DP}} \coloneqq P(a \mid a) - P(b \mid b)$ and $\psi_{\mathsf{TDP}} \coloneqq P(aa) - P(bb)$.

\begin{figure}
	\begin{subfigure}{0.245\linewidth}
		\includegraphics[width=\linewidth]{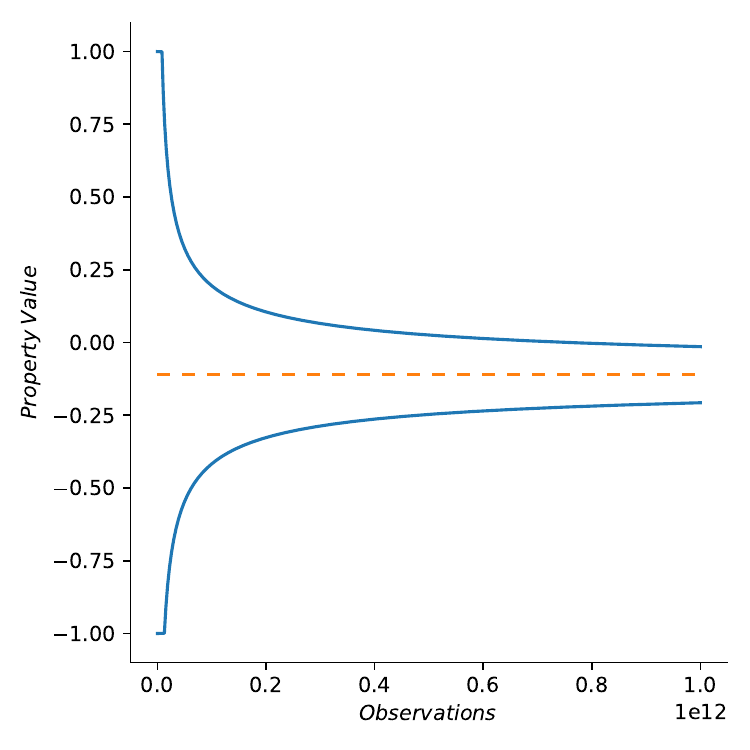}
		\phantomsubcaption
		\label{subfig:a}
	\end{subfigure}
	\begin{subfigure}{0.245\linewidth}
		\includegraphics[width=\linewidth]{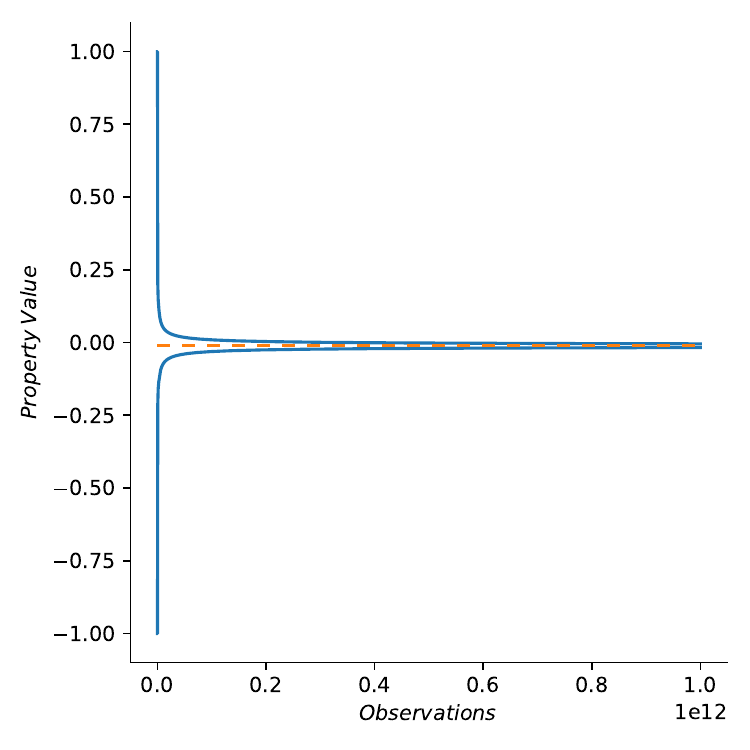}
		\phantomsubcaption
		\label{subfig:b}
	\end{subfigure}
	\begin{subfigure}{0.245\linewidth}
		\includegraphics[width=\linewidth]{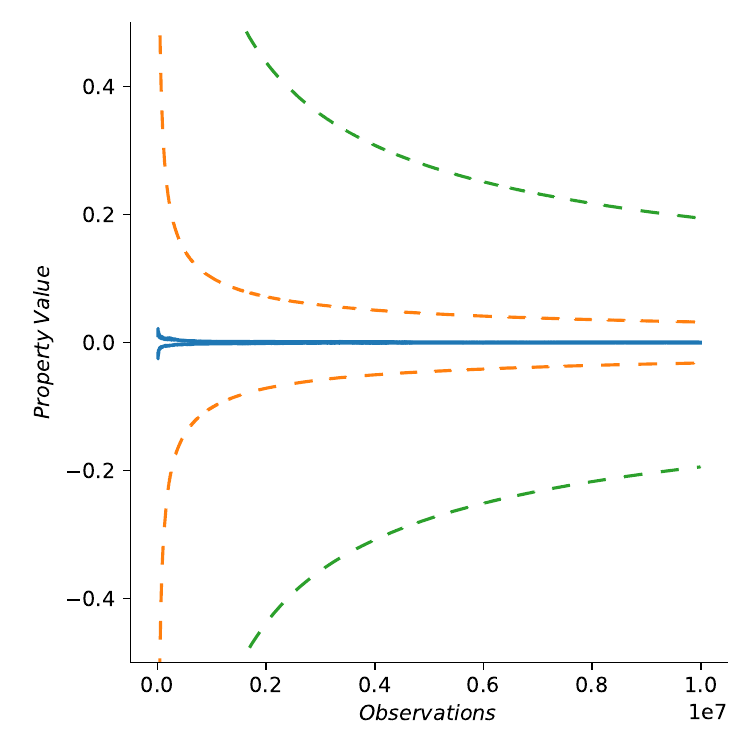}
		\phantomsubcaption
		\label{subfig:c}
	\end{subfigure}
	\begin{subfigure}{0.245\linewidth}
		\includegraphics[width=\linewidth]{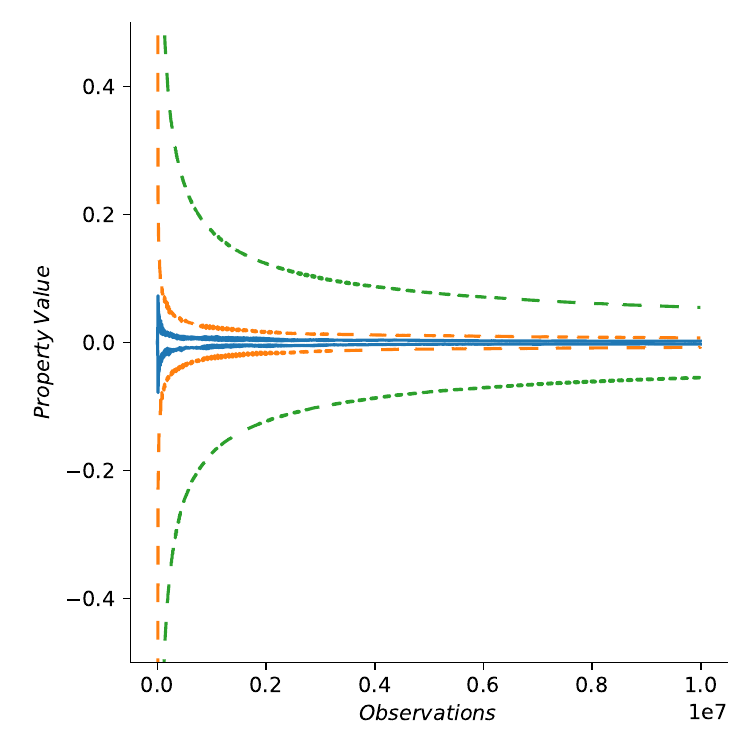}
		\phantomsubcaption
		\label{subfig:d}
	\end{subfigure}
  \caption{Monitoring $\varphi_{\mathsf{DP}}$ (first, third) and $\varphi_{\mathsf{TDP}}$ (second, fourth) on the lending (first, second) and the hypercube (third, fourth) examples. 
  The \underline{first and second} plots show the computed $95\%$-confidence interval (solid) and the true value of the property (dashed) for the lending POMC.
  In reality, the monitor was run for about $7\times 10^8$ steps until the point estimate nearly converged, though the confidence interval was trivial at this point (the whole interval $[-1,1]$), owing to the pessimistic bound $\taumix$.
  In the figure, we have plotted a projection of how the confidence interval would taper over time, had we kept the monitor running.
  The \underline{third and fourth} plots summarize the monitors' outputs over $100$ executions of the hypercube POMC. 
  The solid lines are the max and min values of the point estimates, the dashed lines are the boundaries of all the $95\%$-confidence intervals (among the $100$ executions) with the conservative bound $\taumix$ (green) and the sharper bound $\tau_{\mathsf{true\,mix}}$ (orange) on the mixing time.
  }
\label{fig:lending}
\end{figure}

We empirically evaluated the quality of the confidence intervals computed by our monitor (for $\psi_{\mathsf{DP}}$ and $\psi_{\mathsf{TDP}}$) over a set of $100$ sample runs, and summarize the findings in the third and fourth plots of Fig.~\ref{fig:lending}.
We used $\taumix= 204.94$~steps and $\tau_{\mathsf{true\, mix}}= 7.45$~steps, and we can observe that in both cases, the output with $\tau_{\mathsf{true\,mix}}$ is significantly tighter than with $\taumix$.
Compared to the lending example, we obtain reasonably tight estimate with significantly smaller number of observations, which is due to the smaller bounds on the mixing time.

	\section{Conclusion}
	We generalized runtime verification of algorithmic fairness properties to systems modeled using POMCs and a specification language (BSE) with arithmetic expressions over numerical functions assigning values to observation sequences.
	Under the assumptions of stationary initial distribution, aperiodicity, and the knowledge of a bound on the mixing time, we presented a runtime monitor, which monitors a long sequence of observations generated by the POMC, and after each observation outputs an updated PAC estimate of the value of the given \vbse.
		
	While the new stationarity assumption is important for defining the semantics of the \vbse expressions, the aperiodicity and the knowledge of the bound on the mixing time allow us to use the known McDiarmid's inequality for computing the PAC estimate.
	In future, we intend to eliminate the latter two assumptions, enabling us to use our approach for a broader class of systems.
	Additionally, eliminating the time-homogeneity assumption would also be an important step for monitoring algorithmic fairness of the real-world systems with time-varying probability distributions \cite{henzinger2023runtime}. 
	
	\subsubsection*{Acknowledgments:}
	This work is supported by the European Research Council under Grant No.: ERC-2020-AdG 101020093.
	
	\bibliographystyle{splncs04}
	\bibliography{references}

	\ifarxiv
	\begin{appendix}

\section{Proof of Claims in Sec.~\ref{sec:point estimator for atoms}}

\begin{proof}[Proof of Prop.~\ref{lemma:expectation atom}]
	Let $N = t - n +1$. By definition,
	\begin{align*}
		\expe(\hat{f}(\seq{\rout}))  =  \sum_{\seq{w} \in \Out^{t}} \left(\frac{1}{N} \sum_{i=1}^{N} f(\seq{w}_{i\twodots i+n-1}) \right) \cdot \pr(\seq{w})
	\end{align*}
	where $\pr(\seq{u})$ is defined with respect to the Markov chain. 
	We now express this with respect to the joint probability measure over paths of length $t$ as defined by $\mc$.
	\begin{align*}
		\expe(\hat{f}(\seq{\rout})) & =    \sum_{\seq{q} \in \Q^{t}} \left(\frac{1}{N} \sum_{i=1}^{N} f(\OutFn(\seq{q}_{i\twodots i+n-1})) \right) \cdot \pr(\seq{q}) \\
		&=\sum_{\seq{q} \in \Q^{t}} \frac{1}{N} \sum_{i=1}^{N} f(\OutFn(\seq{q}_{i\twodots i+n-1})) \cdot \pr(\seq{q}) \\
		&=\frac{1}{N} \sum_{i=1}^{N} \sum_{\seq{q} \in \Q^{t}}  f(\OutFn(\seq{q}_{i\twodots i+n-1})) \cdot \pr(\seq{q}).
	\end{align*}
	Fix a particular $i$, i.e. let $A_i\coloneqq \sum_{\seq{q} \in \Q^{t}} f(\OutFn(\seq{q}_{i\twodots i+n-1})) \cdot \pr(\seq{q})$. Split up the internal sum to obtain
	\begin{align*}
		A_i = \sum_{\seq{q}^A \in \Q^{i}}  \sum_{\seq{q}^B \in \Q^{n}}  \sum_{\seq{q}^C \in \Q^{t-n-i}}   f(\OutFn(\seq{q}^B)) \cdot \pr(\seq{q}^A) \cdot \pr(\seq{q}^B \mid q_i^A) \cdot \pr(\seq{q}^C \mid q_n^B).
	\end{align*}
	Now notice that $\sum_{\seq{q}^C \in \Q^{t-n-i}} \pr(\seq{q}^C \mid \seq{q}_n^B)  =1$. Hence, by rearranging the sums we obtain
	\begin{align*}
		A_i &=   \sum_{\seq{q}^B \in \Q^{n}}  f(\OutFn(\seq{q}^B)) \cdot\sum_{\seq{q}^A \in \Q^{i}}  \pr(\seq{q}^A) \cdot \pr(\seq{q}^B \mid \seq{q}_i^A) \\
		&= \sum_{\seq{q}^B \in \Q^{n}}  f(\OutFn(\seq{q}^B))  \cdot \prod_{j=1}^{n-1} M_{q_j^Bq_{j+1}^B} \cdot \sum_{\seq{q}^A \in \Q^{i}} \st_{q_1^A} \cdot \prod_{j=1}^{i-1} M_{q_j^Aq_{j+1}^A} \cdot  M_{q_i^Aq_1^B}.
	\end{align*}
We can use the fact that $\mc$ is in its stationary distribution to express the internal sum as $\st_{q_1^B}$. That is, 
\begin{align*}
	&\sum_{\seq{q}^A \in \Q^{i}} \st_{q_1^A} \cdot \prod_{j=1}^{i-1} M_{q_j^Aq_{j+1}^A} \cdot  M_{q_i^Aq_1^B} = 
	\sum_{q_i^A} \dots \sum_{q_2^A}  \sum_{q_1^A}  \st_{q_1^A} \cdot \prod_{j=1}^{i-1} M_{q_j^Aq_{j+1}^A} \cdot  M_{q_i^Aq_1^B}  \\
	& = \sum_{q_i^A}   M_{q_i^Aq_1^B} \dots \sum_{q_2^A} M_{q_2^Aq_{3}^A}  \cdot \sum_{q_1^A}  \st_{q_1^A}  \cdot M_{q_1^Aq_{2}^A}  =  \sum_{q_i^A}   M_{q_i^Aq_1^B} \dots \sum_{q_2^A} M_{q_2^Aq_{3}^A} \st_{q_2^A}   \\
	& =  \sum_{q_i^A}   M_{q_i^Aq_1^B} \cdot  \st_{q_i^A}  =  \st_{q_1^B}.  \\
\end{align*}
Hence, we obtain 
\begin{align*}
	A_i
	&= \sum_{\seq{q}^B \in \Q^{n}}  f(\OutFn(\seq{q}^B))  \cdot \prod_{j=1}^{n-1} M_{q_j^Bq_{j+1}^B} \cdot  \st_{q_1^B} = \expe(f(\seq{U}))
\end{align*}
where $\seq{U}$ is a random word of length $n$ generated by $\mc$.
Therefore, we obtain
\begin{align*}
	\expe(\hat{f}(\seq{\rout}))=\frac{1}{N} \sum_{i=1}^{N} \expe(f(\seq{U})) =  \expe(f(\seq{U})).
\end{align*}
Moreover, this demonstrates that due to stationarity the expected value of $f$ evaluated on any infix of length $n$ is the same and thus $\expe(\hat{f}(\seq{\rout}))=f(\mc)$.
\end{proof}

\begin{proof}[Proof of Cor.~\ref{lemma:probability atom}]
	This follows directly from Prop.~\ref{lemma:expectation atom} and the observation that $f(\OutFn(\seq{q}^B)) $ removes the probability mass of all the paths of length $n$ whose corresponding words do not belong to $\specset$. Therefore, $\expe(f(\seq{U}))=P(\specset)$ where $\seq{U}$ is as defined in Prop.~\ref{lemma:expectation atom}.
\end{proof}

\section{Proof of Thm.~\ref{thm:soundness of atomic monitor}}

We use a McDiarmid-style inequality to compute the finite-sample confidence bounds.
The version below is a restricted version of the Corollary 2.19 found in \cite{jerison2013general}.

\begin{theorem}[\cite{jerison2013general}]
	\label{thm:mcdiarmid}
	Let $\seq{X}\coloneqq X_1, \dots, X_n$ be an ergodic Markov chain with countable state space $\Q$, unique stationary distribution $\st$, and finite mixing time bounded by $\tau_{mix}$.
	Suppose that some function $f: \Q^n \to \mathbb{R}$ with artiy $n$ satisfies
	\begin{align*}
		f(x)- f(y) \leq \sum_{i=1}^gn  c_i \mathbbm{1}(x_i \neq y_i)
	\end{align*}
	for some $c\in \mathbb{R}^n$ with positive entries. Then for any $\varepsilon>0$
	\begin{align*}
		\pr\left(\left| f(\seq{X}) - \expe(f(\seq{X}))\right|\geq \varepsilon \right) \leq 2\exp\left( - \frac{2\varepsilon^2}{\sqrt{\sum_{i=1}^n c_i^2}^2 \cdot 9\cdot \tau_{mix}} \right)
	\end{align*}

\end{theorem}

To apply Theorem \ref{thm:mcdiarmid} it is required that the observation labels should not interfere with the so-called bounded difference property of the function. 
Below we establish that this requirement is fulfilled by the atoms of \vbse.

\begin{lemma}\label{lemma:bounded difference}
	Let $f\colon\Out^n\to[a,b]$ be a function with fixed $n$, $a$, and $b$, $t\geq n$ be a constant, $\seq{\out}, \seq{\out}' \in \Out^t$ be a pair of observation sequences such that the Hamming distance $|\seq{\out}-\seq{\out}'|_H$ is $1$. 
	Then 
	\begin{align*}
		\hat{f}(\seq{\out})-\hat{f}(\seq{\out}') \leq \frac{\min(t-n+1,n)}{t-n+1} \cdot (b-a).
	\end{align*}
\end{lemma}
\begin{proof}
	Since the Hamming distance is $1$, $\seq{\out}$ and $\seq{\out}'$ differ only in one symbol. 
	We know that $f$ is evaluated on a substring of length $n$. Hence, if the string is sufficiently long only $n$ evaluations of $f$ in $\hat{f}$ (i.e., only $n$ terms in the sum in Eq.~\ref{equ:point estimator}) will be affected, while if the string is short, then only $t-n+1$ evaluations of $f$ will be affected. Therefore, the evaluation of $f$ can differ in the worst case by at most $\frac{\min(t-n+1,n)}{t-n+1} \cdot (b-a)$.
\end{proof}

\begin{lemma}
	\label{lemma:bounded_vale}
	Let $\mc $ be a POMC satisfying Assump.~\ref{ass:stationarity} and \ref{assump:aperiodicity}, $f:\Out^n\to[a,b]$ be a function for a fixed $n$, $a$, and $b$, $t\geq n$ be a constant, and $\seq{\rout} \sim \mc$ be a random observed $\mc$-path of length $|\seq{\rout}|=t$.
	Then
	\begin{align*}
		\pr\left( |f(\seq{\mc}) - \hat{f}(\seq{\rout})|\geq \varepsilon \right) \leq 2\exp\left(- \frac{2\cdot \varepsilon^2 (t-n+1)^2}{t\cdot \min(t-n+1,n) \cdot 9 \cdot \tau_{mix}}\right).
	\end{align*}
\end{lemma}
\begin{proof}
	Notice that $\seq{\out}, \seq{\out}' \in \Out^t$
	\begin{align*}
		|f({\out})-f(\seq{\out}')| \leq \sum_{i=1}^{t} \frac{\min(t-n+1,n)}{t-n+1} \cdot (b-a) \cdot \mathbbm{1}(\out_i \neq \out_i')
	\end{align*}
	Therefore, we conclude that
	\begin{align*}
		\left(\sqrt{\sum_{i=1}^t  \left(\frac{\min(t-n+1,n)}{t-n+1} \cdot (b-a)\right)^2}\right)^2 =  \frac{t\cdot \min(t-n+1,n)^2}{(t-n+1)^2} \cdot (b-a)^2
	\end{align*}
	as required by Theorem \ref{thm:mcdiarmid}.
\end{proof}

\begin{proof}[Proof of Thm.~\ref{thm:soundness of atomic monitor}]
	The soundness claim follows as a consequence of Lem.~\ref{lemma:bounded difference} and Prop.~\ref{lemma:expectation atom}. By combining Theorem \ref{thm:mcdiarmid} and Corollary 2.17 from \cite{jerison2013general} we obtain the result for POMC.
	The computational complexity is dominated by the use of the set of $n$ registers $\seq{w}$ to store the last $n$ sub-sequence of the observed path: allocation of memory for $\seq{w}$ takes $n$ space, and, after every new observation, the update of $\seq{w}$ takes $n$ write operations (Line~\ref{step:shift window}).
\end{proof}

	\end{appendix}
	\fi
\end{document}